\documentclass[10pt]{article}
\usepackage{geometry}
\usepackage[utf8]{inputenc}
\usepackage{amsmath}
\usepackage{amssymb}
\usepackage{amsfonts}
\usepackage{amsthm}
\usepackage{hyperref}
\usepackage{cleveref}
\usepackage{cases}
\usepackage{booktabs}
\usepackage{graphicx}

\newtheorem{theorem}{Theorem}
\newtheorem{proposition}{Proposition}

\newtheorem{lemma}[proposition]{Lemma}

\newtheorem{definition}[proposition]{Definition}

\newcommand{\CC}{\mathcal{C}}
\newcommand{\XX}{\mathcal{X}}
\newcommand{\DD}{\mathcal{D}}
\newcommand{\EE}{\mathbb{E}}
\newcommand{\Yhat}{\hat{Y}}
\newcommand{\round}[1]{\left\lfloor #1 \right\rceil}
\newcommand{\vocab}{\textbf}

\renewcommand{\AA}{\mathcal{A}}

\DeclareMathOperator{\acc}{acc}

\usepackage[round]{natbib}
\hypersetup{
    colorlinks,
    linkcolor={red!50!black},
    citecolor={blue!50!black},
    urlcolor={blue!80!black}
}

\usepackage{xcolor}
\newcommand{\Comments}{1}
\newcommand{\mynote}[3]{\ifnum\Comments=1\textcolor{#1}{#2: #3}\fi}

\title{A No Free Lunch Theorem for Human-AI Collaboration}
\author{Kenny Peng, Nikhil Garg, and Jon Kleinberg\footnote{Kenny Peng, Cornell Tech (\url{klp98@cornell.edu}); Nikhil Garg, Cornell Tech; Jon Kleinberg, Cornell University. We thank Rohan Alur, Kate Donahue, Sophie Greenwood, and Rajiv Movva for valuable discussion and feedback. Nikhil Garg is supported by NSF CAREER IIS-2339427 and Cornell Tech Urban Tech Hub and Amazon research awards.}}
\date{}

\begin{document}

\maketitle

\begin{abstract}
The gold standard in human-AI collaboration is \textit{complementarity}---when combined performance exceeds both the human and algorithm alone. We investigate this challenge in binary classification settings where the goal is to maximize 0-1 accuracy. Given two or more agents who can make calibrated probabilistic predictions, we show a ``No Free Lunch''-style result. Any deterministic collaboration strategy (a function mapping calibrated probabilities into binary classifications) that does not essentially always defer to the same agent will sometimes perform worse than the \textit{least accurate} agent. In other words, complementarity cannot be achieved ``for free.'' The result does suggest one model of collaboration with guarantees, where one agent identifies ``obvious'' errors of the other agent. We also use the result to understand the necessary conditions enabling the success of other collaboration techniques, providing guidance to human-AI collaboration.
\end{abstract}

\section{Introduction}

Many important decisions depend---in large part---on prediction. Doctors decide if a patient should undergo a procedure by predicting if the operation will succeed. Judges decide whether or not to grant bail by predicting if the defendant will reoffend. Loan officials decide whether or not to offer a loan by predicting if the loan will be repaid. In all of these tasks, algorithmic predictions are now commonly incorporated into the decision-making process. Still, humans remain a central part of each of these settings, and often have the final say. The hope of human-AI collaboration is that a human and an algorithm can leverage their unique strengths to make predictions that are more accurate than either alone. This standard has been called \textit{complementarity}.

The present work investigates the conditions under which complementarity can be guaranteed when making binary classifications, where the goal is to maximize 0-1 accuracy (minimize the number of misclassifications). We consider a setup in which two or more agents can make calibrated probabilistic predictions on a shared task. These predictions may differ---for example, due to differences in the information available to each agent. Each agent can use their probabilistic predictions to make binary classifications, each achieving some level of accuracy. We ask: Is there a way to combine each agent's calibrated predictions to produce binary classifications that are guaranteed to be at least as accurate, and sometimes \textit{more accurate}, than every individual agent? In other words, if agents are calibrated, can we achieve complementarity ``for free''?

Our main result answers this question, mostly in the negative. In fact, \Cref{thm:main} shows that it is difficult to ensure a much lower bar: producing binary classifications that are always at least as accurate as \textit{the worst} individual agent. As we will show---except under very narrow circumstances---the only way to guarantee that a collaboration does not perform worse than the worst agent is to always defer to a single agent (in which case the standard is trivially met). In other words, substantive collaboration in our setup must at times come at a significant cost. There is no ``free lunch.''

Let us be more concrete, if still somewhat informal. Consider $n$ agents, who, given an input $x$, produce calibrated probabilistic predictions $P_1(x), P_2(x), \cdots, P_n(x)\in [0,1].$ By calibrated, we mean that among inputs for which an agent predicts a positive label with probability $p$, the true proportion of positive labels is in fact $p$. For agent $k$, given their calibrated prediction $P_k(x),$ the optimal 0-1 classification to maximize accuracy is obtained by a threshold rule: predict $1$ if $P_k(x) > 0.5$ and $0$ if $P_k(x) < 0.5.$ Using $\round{\cdot}$ to denote the rounding operator, the optimal classification is $\round{P_k(x)}.$ In this way, each individual agent can achieve some level of accuracy on their own. A collaboration strategy is a way to combine the calibrated predictions $P_1(x),P_2(x),\cdots,P_n(x)$ into a 0-1 classification. Therefore, a collaboration strategy is defined as a function $\CC:[0,1]^n\rightarrow \{0,1\}.$ There are a number of intuitive collaboration strategies: one approach is to round the average of predicted probabilities; another is to take a majority vote of each agent's classifications; yet another is to defer to the classification of the most confident agent (i.e., the agent whose probabilistic prediction is furthest from $\frac{1}{2}$). Any given collaboration strategy also achieves an accuracy, which can be compared to the accuracy of individual agents.

We call a collaboration strategy $\CC$ \textbf{reliable} if it is always at least as accurate as the \textit{least accurate} agent. We call $\CC$ \textbf{non-collaborative} if there exists $k\in [n]$ such that for all $(p_1,p_2,\cdots,p_n)\in (0,1)^n$, $\CC(p_1,p_2,\cdots,p_n) = \round{p_k}.$ In other words, $\CC$ is non-collaborative if and only if it always defers to the same agent, except in the special case when another agent is certain in their prediction (i.e., predicts exactly $0$ or $1$). We may then state our main result.

\begin{theorem}
Every reliable collaboration strategy 
is non-collaborative.
\end{theorem}

Consequently, essentially any collaboration strategy that can sometimes achieve complementarity must at other times perform worse than \textit{all} agents. For example, it shows that none of the collaboration strategies described above---averaging probabilities, majority vote, deferring to the most confident agent---are reliable; each of these collaboration strategies sometimes perform worse than the least accurate agent. \Cref{thm:main} is not, however, a statement about these particular collaboration strategies---rather, it applies to the entire space of collaboration strategies $\CC:[0,1]^n\rightarrow \{0,1\}.$ In this way, \Cref{thm:main} is a ``No Free Lunch''-style result \citep{wolpert1997no}. Wolpert and Macready showed that in optimization problems, ``if an algorithm performs well on a certain class of problems then it necessarily pays for that with degraded performance on the set of all remaining problems.'' Like Wolpert and Macready's classical result, our result implies that further structure in the collaboration setting must be leveraged or assumed to obtain guarantees. (There is also a sense in which our result is reminiscent of Arrow's Impossibility Theorem \citep{arrow1950difficulty}, which shows that the only way to aggregate votes while satisfying certain axioms is by deferring to a single dictator. However, our results fundamentally differ, in that Arrow focuses on aggregating ranked preference lists, whereas we focus on aggregating continuous predictions.)

Note, however, that \Cref{thm:main} leaves a small window for collaboration: if a (calibrated) agent is certain, it is clearly possible (and optimal) to defer to that agent. This points to a possible model of successful human-AI collaboration, in which one agent is only in charge of identifying ``obvious'' errors of the other agent. In contrast, it is not sufficient, however, for an agent to be very confident in their prediction (i.e., predicting close to probability $0$ or $1$). The agent must be certain. At a high level, the reason is because an agent's predicted probability is not (generally) calibrated after conditioning on the other agents' predictions. The exception is when an agent is fully certain, in which case no outside information can alter this certainty.

Many results in computer science and economics---experts, ensemble prediction methods, Condorcet's Jury Theorem, and recent human-AI collaboration techniques---demonstrate that collaboration in prediction is often possible when further structure is imposed. \Cref{thm:main} can be used to shed light on what conditions enable these approaches to succeed. In \Cref{sec:discussion}, we survey this literature, identifying two conditions that distinguish these approaches from our setting, both of which enables success: independence in predictions, or (learned) knowledge of the joint distribution of agent predictions and outcomes. Neither condition is satisfied in the setup of \Cref{thm:main}. Our result thus suggests that conditions like these are necessary to ensure successful collaboration. 

In particular, our setup is reminiscent of typical implementations of human-AI collaboration where humans are shown a probabilistic algorithmic prediction to incorporate into their decision (e.g., \citet{shin2021ai,cabrera2023improving}). In these situations, neither independence nor knowledge of the joint distribution is guaranteed. Existing empirical evidence suggests that these kinds of implementations do not typically result in complementarity, even in laboratory settings (e.g., \citet{green2019principles,lai2019human,kiani2020impact}). \cite{vaccaro2024combinations} provide a meta-analysis demonstrating that most studies of human-AI collaboration do not document complementarity. It should be noted that in real-world settings, human behavioral biases and cognitive limitations further complicate human-AI collaboration \citep{buccinca2021trust, bhatt2021uncertainty, bondi2022role}. \Cref{thm:main} suggests that even should humans not suffer from these limitations, there remain fundamental challenges in combining expertise in prediction problems.

To prove \Cref{thm:main} (which we do in \Cref{sec:proof}), it suffices to show that for any collaboration strategy $\CC:[0,1]^n\rightarrow\{0,1\}$ that is not non-collaborative, it is possible to construct a setting in which $\mathcal{C}$ is less accurate than all individual agents. The high-level approach is to use the violation of the non-collaborative condition to construct a set of inputs on which we know the collaboration strategy's behavior: specifically, for every $k\in [n]$, there must exist $(p_1, p_2, \cdots, p_n)$ that $\CC$ classifies differently than $k$. We use this to construct a setting in which $\CC$ performs at least as bad as every agent and strictly worse than $k$. After doing this for each $k$, we can then ``glue together'' the resulting settings to obtain the desired adversarial example. Intuitively, our result and proof leverage the insight that what matters for collaboration is the \textit{joint} distribution of agent predictions: conditional on the predictions of other agents, when is an agent accurate? Despite calibration implying that each agent has a strong sense of when they are more or less confident, calibration alone does not reveal sufficient information about interactions between agent predictions---more is needed to know when to defer to one agent or another. 
\section{A No Free Lunch Theorem}\label{sec:main-result}
A binary classification problem can be represented as a distribution $\DD$ over $\XX\times \{0,1\}$ where $\XX$ is the input space. A \vocab{classifier} is a function $\Yhat: \XX\rightarrow \{0,1\}$. The 0-1 \vocab{accuracy} of a classifier is given by
\begin{equation}
    \EE_{(X,Y)\sim \DD} |\Yhat(X) - Y|.
\end{equation}
A \vocab{predictor} is a function $P: \XX\rightarrow [0,1]$. A predictor $P$ is \vocab{calibrated} on $\DD$ if
\begin{equation}
    \Pr_{(X,Y)\sim \DD}[Y=1\,|\,P(X)=p] = p
\end{equation}
for all $p\in [0,1]$ (more precisely, for all $p\in \text{Image}(P)$). 

We may now define collaboration settings.
\begin{definition}A \textbf{collaboration setting} is an ordered tuple 
\begin{equation}
    S=(\DD, P_1, \cdots, P_n),
\end{equation}
where $\DD$ is a probability distribution over $\XX\times \{0,1\}$ and $P_1,\cdots,P_n$ are each calibrated predictors on $\DD$. 
\end{definition}

The predictor $P_i$ induces the classifier $\Yhat_i: x\mapsto \round{P_i(x)},$ where $\round{\cdot}$ is the rounding operator (without loss of generality, set $\round{0.5}=1$). $\Yhat_i$ achieves 0-1 accuracy
\begin{equation}
    \acc_i(S) := \EE_{(X,Y)\sim \DD}[\max\{P_i(X), 1 - P_i(X)\}].
\end{equation}

A collaboration strategy is a way to combine the predicted probabilities of each agent to make a classification.
\begin{definition}
A \textbf{collaboration strategy} is a deterministic function $\CC: [0,1]^n \rightarrow \{0,1\}.$
\end{definition}
Here, $\CC$ should be interpreted as a function that takes in $n$ predicted probabilities and returns a 0-1 classification. Given a collaboration setting $S=(\DD, P_1, \cdots, P_n)$, a collaboration strategy $\CC$ induces the classifier
\begin{equation}
    \Yhat_\CC: \XX \rightarrow \{0,1\},\quad x \mapsto \CC(P_1(x),\cdots,P_n(x)).
\end{equation}
Let $\acc_\CC(S)$ denote the 0-1 accuracy of $\Yhat_\CC$ on $\DD$. In other words, for each $x\in \XX$, each agent $i$ produces a predicted probability $P_i(x).$ The collaboration strategy aggregates these predictions into a binary classification. 

\begin{definition}
    A collaboration strategy $\CC:[0,1]^n\rightarrow \{0,1\}$ is \textbf{reliable} if $\acc_\CC(S) \ge \min_{i\in [n]} \acc_i(S)$ for all collaboration settings $S$.
\end{definition}
In words, a collaboration strategy is reliable if and only if it always performs at least as well as the \textit{least accurate} agent.

\begin{definition}\label{def:non-collaborative}
    A collaboration strategy $\CC:[0,1]^n\rightarrow \{0,1\}$ is \textbf{non-collaborative} if there exists $k\in [n]$ and $\alpha\in \{0,1\}$ such that 
    \begin{equation}
    \CC(p_1,p_2,\cdots,p_n) = 
        \begin{cases}
        \round{p_k} &\quad \text{$p_k\neq \frac{1}{2}$} \\
            \alpha &\quad \text{$p_k=\frac{1}{2}$}
        \end{cases}
    \end{equation}
    for all $(p_1,p_2,\cdots,p_n)\in (0,1)^n$.
\end{definition}

A collaboration strategy is non-collaborative if it essentially always defers to the classification of a single agent $k\in [n]$. The definition admits two exceptions. First, when $p_k = \frac{1}{2}$, the collaboration may select either $0$ or $1$, but must always select the same such value; this choice has no effect on the accuracy of the classifier, so this exception can be essentially ignored. Second, the collaboration strategy need not select $\round{p_k}$ when $(p_1,p_2,\cdots,p_n)\notin (0,1)^n$---i.e., when $p_i\in \{0,1\}$ for some $i\in [n]$; this exception is somewhat more interesting, and we will return to it after stating our main result.

\setcounter{theorem}{0}
\begin{theorem}\label{thm:main}
    Every reliable collaboration strategy is non-collaborative.

\end{theorem}

\Cref{thm:main} implies that the search for ``free'' complementarity in human-AI collaboration is futile. Any collaboration strategy that is reliable (performs no worse than the \textit{worst} agent) is non-collaborative (essentially always defers to the same one agent).

While \Cref{thm:main} implies that \textit{any} reliable collaboration strategy generally must defer to the same agent, it allows for one exception. In particular, taking note of \Cref{def:non-collaborative}, if some other agent is entirely confident in their prediction ($p_i\in \{0,1\}$ for some $i$), it is not necessary to always defer to the same agent $k$. Indeed, it is always ``reliable'' (and, in fact, optimal) to defer to the prediction of an agent who is entirely confident, since agents are calibrated---guaranteeing that on the set of points for which the agent predicts probability 1, all points are truly positive. This suggests a simple collaboration strategy: always defer to a fixed agent, except when another agent is entirely confident. For example, either the human or algorithm can be assigned ``primary'' decision-making power, and the other party can be tasked only with overriding obvious mistakes.

This exception illustrates an underlying intuition behind \Cref{thm:main}: conditional on other agents' predictions, a given agent's prediction is no longer calibrated. Only in the specific case when an agent is \textit{certain} in their prediction, can they be confident. Otherwise, even if an agent predicts probability $0.95$, for example, deferring to that agent is not guaranteed to be a good idea in all situations---given the predictions of other agents, and conditional on the choice of deferring to the agent, the prediction of $0.95$ may be far from calibrated.

\section{Implications for Human-AI Collaboration}\label{sec:discussion} 

In this section, we use \Cref{thm:main} to better understand the conditions that enable effective human-AI collaboration. We begin by discussing numerous settings in ML, human-AI collaboration, and beyond, in which collaboration has been shown to be possible (often, with guarantees). We then identify two common features of these “success stories”---features which are not present in the setup of \Cref{thm:main}. We then argue that common implementations of human-AI collaboration also lack these features, and suggest a path forward. Like Wolpert and Macready's ``No Free Lunch Theorem,'' a primary use of \Cref{thm:main} is in clarifying the additional structure needed to ensure successful prediction.

\subsection{Successful Collaborations}
There are many lines of work that are fundamentally about collaboration in classification tasks, each of which---unlike us---obtain positive results. For example, the machine learning literature is ripe with such results. Combining expert predictions is a basic problem in online learning theory (see \citet{blum2005line} for an overview). There, it has been shown, for example, that expert predictions can be combined to perform better than the best linear combination of experts \citep{littlestone1991line}. The idea of mixing expert predictions is also seen in the literature on ensemble classifiers, including boosting methods and random forests. Pivoting, Condorcet's jury theorem \citep{condorcet1785essai} provides a collaboration-based view of voting theory: when individual jurors are biased towards the correct decision, the majority vote is more accurate than individual votes. Finally, recent work on human-AI collaboration has presented numerous approaches to achieving complementarity \citep{madras2018predict, donahue2022human, alur2024distinguishing}. Why do each these methods work, and why do they have guarantees? We suggest two distinct features behind these successful collaborations.

\paragraph{Leveraging Independence: Condorcet's Jury Theorem, Wisdom of Crowds, Random Forests.}

Condorcet's jury theorem \citep{condorcet1785essai} states that when individual jurors are biased towards the correct decision (i.e., vote in that direction independently with probability $p>\frac{1}{2}$), the majority vote is more accurate than any individual vote. Indeed, the idea of aggregating independent signals appears in a much broader literature studying information aggregation and the ``wisdom of crowds.'' A key distinction between \Cref{thm:main} and these settings is a lack of ``independence'' in our setting; agents need not make predictions ``independently'' in our setup. This also appears to be the key distinction between our setting and that of majority vote ensemble classifiers in ML such as random forests, which rely on some amount of independence in how decision trees are constructed \citep{breiman2001random}. Independence cannot be guaranteed in human-AI collaboration in this way.

\paragraph{Leveraging Learning: Experts, Boosting.}

\Cref{thm:main} contrasts with the longstanding machine literature in machine learning that shows how multiple ``expert'' predictions can be effectively combined \citep{blum2005line}. In fact, the idea of combining expert advice has formed a fruitful baseline intuition for how to build effective ML algorithms more broadly, such as in boosting \citep{schapire1999brief}. A crucial aspect of these methods is the process of \textit{learning} which experts to trust, and when and how much to trust them. (Note that while both random forests and boosting are considered ensemble methods---methods that combine predictions---the former relies on independence of predictions and the latter on learning joint behavior.) This learning process is absent in the setup of \Cref{thm:main}. While agents have strong information about their own predictions (calibration), they do not know direct information about joint behavior.

\subsection{Human-AI Collaboration}
We have described two general approaches to ensuring effective collaboration: independence and learning. Since we cannot generally ensure that a human and algorithm produce independent estimates, we focus on the potential of the latter approach.

Learning enables collaboration by understanding the joint behavior of agents. \Cref{thm:main} itself illustrates this point in one setting: when an agent is entirely certain in their prediction. In such a setting, the relevant ``joint'' information is fully understood: regardless of what the other agents predict, it is safe to defer to the agent. This connects more generally to recent work showing that complementarity is achievable exactly when there are subsets of the domain in which each agent has an advantage \citep{donahue2022human}. The approach of ``overriding only when an agent is certain'' can be viewed in this framing, in which there is a designated region in which one agent has a clear advantage (due to their full certainty), but otherwise, the other agent is deferred to. However, as \Cref{thm:main} implies, the regions in which each agent has an advantage cannot in general be determined \textit{a priori} from only their predictions in each region (even if calibrated probabilities intuitively give a measure of effectiveness). Thus, implementations of such collaboration models must reason more directly with whether or not one agent is more equipped to handle a given subset of the feature space. One way in which to do this is by performing additional training using joint information. Indeed, idea has been taken up by ``learning to defer'' approaches (e.g., \citet{madras2018predict, mozannar2022teaching}). Similarly, \citet{alur2024distinguishing} introduce a method to identify subsets of inputs that are indistinguishable to an agent, in which case signal from the prediction of the other agent can then be leveraged to improve predictions overall. These approaches require data from the joint distribution of agent predictions and outcomes. \Cref{thm:main} suggests that such data is necessary to obtain guarantees.
\section{Proof of \Cref{thm:main}}\label{sec:proof}

Before proceeding to the proof of \Cref{thm:main}, we begin by establishing some basic language and tools with which to analyze and construct collaboration settings.

\subsection{Preliminaries}

\paragraph{Correctness and Agreement.}
We first introduce basic language to describe the performance of agents and collaboration strategies.
\begin{definition}\label{def:correct}
For a collaboration setting $(\DD, P_1,\cdots,P_n)$ and $x\in \XX,$ we say that
\begin{itemize}
    \item agent $i$ is \textbf{correct} on $x$ if $$\Yhat_i(x) = \round{\Pr_{(X,Y)\sim \DD}[Y=1\,|\,X=x]}$$
    \item agent $i$ is \textbf{incorrect} on $x$ if $$\Yhat_i(x) \neq \round{\Pr_{(X,Y)\sim \DD}[Y=1\,|\,X=x]}$$
    \item agents $i$ and $j$ \textbf{agree} on $x$ if $$\Yhat_i(x) = \Yhat_j(x)$$
    \item agents $i$ and $j$ \textbf{disagree} on $x$ if $$\Yhat_i(x) \neq \Yhat_j(x).$$
\end{itemize}
For each of these statements, we can replace $i$ or $j$ with a collaboration strategy $\CC$.
\end{definition}

For example, if $\Pr_{(X,Y)\sim \DD}[Y=1\,|\,X=x] = 0.75,$ then $i$ is correct on $x$ if and only if $\Yhat_i(x) = 1.$ (This is the correct classification to maximize 0-1 accuracy.) Using the language established in \Cref{def:correct}, we can make some simple observations about accuracies. For example, if $i$ is correct on $x$ whenever $j$ is correct on $x$, this implies that $\acc_i(S)\ge \acc_j(S).$ If there furthermore is some $x$ for which $i$ is correct but $j$ is incorrect, and where $\Pr_{(X,Y)\sim \DD}[X=x]\neq 0$, this implies that $\acc_i(S)>\acc_j(S).$

\paragraph{Combining collaboration settings.}
Having established some basic language with which to describe and analyze the performance of experts and collaboration strategies on a collaboration setting, we now establish a basic tool for constructing collaboration settings ``piece by piece.''

\begin{proposition}\label{prop:linear-combination} \textbf{Linear combinations of settings.}
    Consider $\ell$ collaboration settings $S_1, \cdots, S_\ell$. Then for all $(\lambda_1, \cdots, \lambda_\ell)\in \Delta^\ell$, there exists a collaboration setting $S$ such that
    \begin{align}
        \acc_i(S) &= \sum_{m=1}^\ell \lambda_m \acc_i(S_m)\label{eq:monstera-1}\\
        \acc_\CC(S) &= \sum_{m=1}^\ell \lambda_m \acc_\CC(S_m)\label{eq:monstera-2}
    \end{align}
    for all $i\in [n]$ and collaboration strategies $\CC$.
\end{proposition}

\begin{proof}
    Let $S_m = (\DD_m, P_{1,m}, \cdots, P_{n, m})$ for $m\in [\ell],$ where $\DD_m$ is a distribution over $\XX_m\times \{0,1\}.$ Then consider the collaboration setting $S = (\DD, P_1, \cdots, P_n)$, where $\DD$ is a distribution over $\XX\times \{0,1\}$ for $\XX := \bigcup_{m=1}^\ell \{(m, x):x\in \XX_m\}$. Specifically, we define $\DD$ such that 
    \begin{equation}
        \Pr_{(X,Y)\sim \DD}[X = (m,x)] = \lambda_m\Pr_{(X_m,Y_m)\sim \DD_m}[X_m=x]
    \end{equation}
    and
    \begin{align}
        \Pr_{(X,Y)\sim \DD}[Y=1 | X = (m,x)] =  \Pr_{(X_m,Y_m)\sim \DD_m}[Y_m=1 | X_m=x].
    \end{align}
    
    $\DD$ is essentially the distribution obtained by first randomly sampling $m\in [\ell]$ with probability $\lambda_\ell$ and then sampling from $\DD_m$. Now define $P_i$ such that $P_i(m,x) = P_{i,m}(x)$ for all $x\in \XX_m.$ $P_i$ is calibrated since $P_{i,m}$ is calibrated for all $m\in [\ell].$ By inspection, $S$ satisfies \eqref{eq:monstera-1} and \eqref{eq:monstera-2}.
\end{proof}

\paragraph{Building Calibrated Predictors from Partitions.} 
Finally, given a distribution $\DD$ over $\XX\times \{0,1\}$, we show how partitions of the input space $\XX$ induce calibrated predictors. Indeed, let $\AA_i$ be a partition of $\XX$. Then $\AA_i$ induces a calibrated predictor $P_i$, where for $x\in A \in \AA_i$,
\begin{equation}
    P_i(x) := \Pr_{(X,Y)\sim \DD}[Y=1\,|\,X\in A].
\end{equation}
Here $P_i$ is the Bayes-optimal predictor given a ``coarsening'' of the input space into the partitions $\XX.$ In this way, a collaboration setting may also be identified by an ordered tuple $(\DD, \AA_1, \cdots, \AA_n).$ This approach is central to the subsequent proofs.

\subsection{Main Proof}
In the remainder of this section, we prove \Cref{thm:main} in full. We first rewrite \Cref{thm:main} in an equivalent formulation.

\setcounter{theorem}{0}
\begin{theorem}
For a collaboration strategy $\CC$, $\acc_\CC(S) \ge \min_{i\in [n]}\acc_i(S)$ for all collaboration settings $S$ if and only if there exists $k\in [n]$ and $\alpha \in \{0,1\}$ such that for all $(p_1,p_2,\cdots,p_n)\in (0,1)^n$:

\begin{enumerate}
    \item[(i)] If $p_k\neq \frac{1}{2},$ then
    $\CC(p_1,p_2,\cdots,p_n)=\round{p_k}.$
    \item[(ii)] If $p_k=\frac{1}{2}$, then
        $\CC(p_1,p_2,\cdots,p_n) = \alpha.$
\end{enumerate}
\end{theorem}

The high level plan for the proof is to first show that there must exist $k$ such that condition (i) holds. Then, given that there must exist $k$ such that (i) holds, we show that it must be the case that (ii) also holds (for the same $k$). The heart of the proof is in showing the first step, formalized in the proposition below.

\begin{proposition}\label{prop:part1}
    For a collaboration strategy $\CC$, if $\acc_\CC(S)\ge \min_{i\in [n]}\acc_i(S)$ for all collaboration settings $S$, then there must exist $k\in [n]$ such that for all $(p_1,p_2,\cdots,p_n)\in (0,1)^n$ where $p_k\neq \frac{1}{2},$ $\CC(p_1,p_2,\cdots,p_n)=\round{p_k}.$ 
\end{proposition}

To show \Cref{prop:part1}, suppose for sake of contradiction that there does not exist $k\in [n]$ satisfying this property. This means that for every $k\in [n]$, there must exist some tuple $(p_1,p_2,\cdots,p_n)\in (0,1)^n$ where $p_k\neq \frac{1}{2},$ such that $\CC(p_1,p_2,\cdots,p_n)\neq \round{p_k}.$ We show in  \Cref{lem:counterexample} below that the existence of such a tuple implies that there is a collaboration setting $S_k$ such that $\acc_k(S_k)>\acc_\CC(S_k)$ and $\acc_i(S_k)\ge \acc_\CC(S_k)$ for all $i\in [n]\setminus \{k\}.$ Since we can construct such an $S_k$ for all $k\in [n]$, \Cref{prop:linear-combination} implies the existence of a collaboration setting $S$ such that
\begin{equation}
     \acc_\CC(S) = \sum_{k=1}^n \frac{1}{n} \acc_\CC(S_k) < \sum_{k=1}^n \frac{1}{n} \acc_i(S_k) = \acc_i(S)
\end{equation}
for all $i\in [n],$ providing the desired contradiction. Therefore, it suffices to show \Cref{lem:counterexample}.

\begin{lemma}\label{lem:counterexample}
    Consider a collaboration strategy $\CC$. Suppose that there exists a tuple $(p_1,p_2,\cdots,p_n)\in (0,1)^n$ where $p_k\neq \frac{1}{2}$ and $\CC(p_1,p_2,\cdots,p_n)\neq \round{p_k}.$ Then there exists a collaboration setting $S_k$ such that
    \begin{itemize}
        \item[(i)] $\acc_k(S_k)>\acc_\CC(S_k)$,
        \item[(ii)]
        $\acc_i(S_k)\ge \acc_\CC(S_k)$ for all $i\neq k$.
    \end{itemize}
\end{lemma}

\begin{proof}

We first define the collaboration setting $S_k(\DD, \AA_1, \cdots, \AA_n).$ Define the input space $\XX = \{0, 1, \cdots, n\},$ and for all $i\in [n]$, let $\AA_i$ be the partition comprising the set $\{0, i\}$ together with the singleton sets $\{j\}$ for $j\notin\{0,i\}.$ ($\XX$ and $\AA_i$ are depicted in \Cref{fig:partition}.)  Define a distribution $\DD$ over $\XX\times \{0,1\}$ in the following manner.
\begin{figure}
    \begin{center}
        \includegraphics[width=10cm]{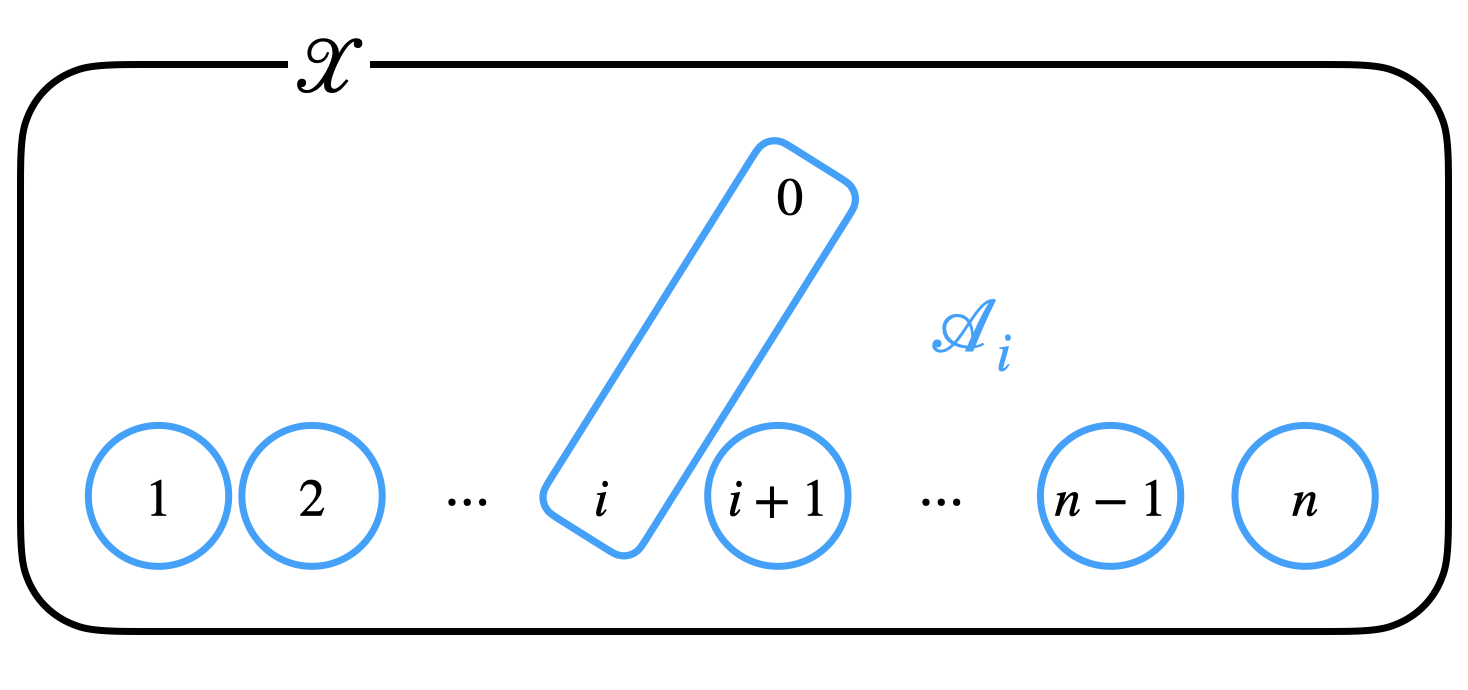}
        \caption{An illustration of a collaboration setting constructed in the proof of \Cref{prop:part1}: the input space $\XX = \{0, 1, 2, \cdots, n\}$ and the partition $\AA_i$ (comprised of $\{0, i\}$ and the remaining singleton sets $\{j\}$ for $i\notin \{0, i\}$). In this setting, agent $i$ is always correct for the inputs $x\notin \{0,i\}$, since $P_i(x)$ is exactly $\Pr[Y=1|X=x]$ on these points. The full collaboration setting used in \Cref{prop:part1} consists of combining $n$ such settings---one for each agent $k\in [n]$. Each such setting $S_k$ is constructed such that the collaboration strategy $\CC$ performs strictly worse than agent $k$ and no better than the remaining agents.}
        \label{fig:partition}
    \end{center}
\end{figure}
Set
\begin{equation}\label{eq:scallop}
    \Pr_{(X,Y)\sim \DD}[Y=1\,|\,X=x] = \begin{cases}
        1 - \CC(p_1,\cdots,p_n) &\,\, \text{for }x=0\\
        \CC(p_1,p_2,\cdots,p_n)&\,\, \text{for }x\in [n].
    \end{cases}
\end{equation} 
Furthermore, when $\CC(p_1,\cdots,p_n)=0,$ set
\begin{equation}
    \Pr_{(X,Y)\sim \DD}[X=x] =
    \begin{cases}
         \frac{1}{1 + \sum_{j\in [n]} (1-p_j)/p_j} & \text{for }x=0\\
        \frac{(1-p_i)/p_i}{1 + \sum_{j\in [n]} (1-p_j)/p_j} & \text{for }x\in [n].
    \end{cases}
\end{equation}
When $\CC(p_1,\cdots,p_n)=1,$ set
\begin{equation}
    \Pr_{(X,Y)\sim \DD}[X=x] = \begin{cases}
         \frac{1}{1 + \sum_{j\in [n]} p_j/(1-p_j)} & \text{for }x=0\\
        \frac{p_i/(1-p_i)}{1 + \sum_{j\in [n]} p_j/(1-p_j)} & \text{for }x\in [n].
    \end{cases}
\end{equation}
This completes the construction of the collaboration setting $S_k.$ $S_k$ satisfies the key property that for all $i\in [n],$ $\Pr_{(X,Y)\sim \DD}[Y=1\,|\,X\in \{0, i\}] = p_i.$ This implies that $P_i(0) = P_i(i) = p_i$, which further implies
a central feature of this construction: that
\begin{equation}
    \Yhat_\CC(0) = \CC(P_1(0), \cdots, P_n(0)) = \CC(p_1,\cdots,p_n).
\end{equation}
Meanwhile, by construction \eqref{eq:scallop},
    $\Pr[Y=1\,|\,X=0] = 1 - \CC(p_1,\cdots,p_n).$
Therefore, $\CC$ is incorrect on $0$ (using the terminology established in Definition \ref{def:correct}). On the other hand, for all $i\in [n]$, agent $i$ is correct on all $x\in [n]\setminus \{i\}$ since agent $i$ is correct on all singletons in $\AA_i$. The correctness of agents $i$ and the collaboration strategy $\CC$ can be summarized as follows:
\begin{center}
    \begin{tabular}{cccc}
        \toprule
        & $x=0$ & $x=i$ & $x\in [n]\setminus i$ \\
        \midrule
        agent $i$ & ? & ? & correct\\
        $\CC$ & incorrect & ? & ?\\
        \bottomrule
    \end{tabular}
\end{center}
Now note that $\Yhat_i(0)=\Yhat_i(i)$ and $\Pr[Y=1\,|\,X=0] = 1 - \Pr[Y=1\,|\,X=i],$ so $i$ must be correct on either $0$ or $i$. To complete the proof, we would like to show that $\acc_\CC(S_k)< \acc_k(S_k)$ and $\acc_\CC(S_k)\le \acc_i(S_k)$ for $i\neq k.$ To show these inequalities, since agent $i$ is always correct on $j\notin \{0,i\}$, it suffices to analyze the accuracy of the classifiers on $x=0$ and $x=i$. This can be handled in two cases:
\begin{itemize}
\item When $\CC$ agrees with agent $i$ on $x=0$, agent $i$ is incorrect on $x=0$ and correct on $x=1$. In this case, agent $i$ is correct whenever $\CC$ is correct, so $\acc_\CC(S_k)\le \acc_i(S_k).$
\item When $\CC$ disagrees with agent $i$ on $x=0$, agent $i$ is correct on $x=0$ and incorrect on $x=1$; $\CC$ is incorrect on $x=0$, and perhaps correct on $x=1$. Therefore, agent $i$ is at least as accurate than $\CC$ on $\{0,i\}$ if $\Pr[X=0]\ge \Pr[X=i]$, and is strictly more accurate than $\CC$ if $\Pr[X=0]> \Pr[X=i]$.
\\
\\
If $p_i=\frac{1}{2},$ then $\Pr[X=0]=\Pr[X=i].$ Therefore, agent $i$ is at least as accurate as $\CC$ on $\{0,i\}$, so $\acc_\CC(S_k)\le \acc_i(S_k).$ 
\\
\\
If $p_i\neq \frac{1}{2}$, since $i$ makes the optimal classification with respect to the set $\{0,i\}$ (recalling the partition $\AA_i$), $\Pr[X=0]>\Pr[X=i].$ Therefore, agent $i$ is strictly more accurate than $\CC$ on $\{0,i\}$, so $\acc_\CC(S_k) < \acc_i(S_k).$ 
\end{itemize}
In every case, $\acc_\CC(S_k)\le\acc_i(S_k).$ Furthermore, by assumption, $k$ disagrees with $\CC$ on $x=0$ and $p_k \neq \frac{1}{2}$, so $\acc_\CC(S_k)<\acc_k(S_k).$
\end{proof}

We now show the second component of the proof, handling the case where $p_k=\frac{1}{2}.$
\begin{proposition}\label{prop:part2}
    Consider a collaboration strategy $\CC$ such that there exists $k\in [n]$ such that for all $(p_1,p_2,\cdots,p_n)\in (0,1)^n$ where $p_k\neq \frac{1}{2},$ $\CC(p_1,p_2,\cdots,p_n)=\round{p_k}.$ Then, if $\acc_\CC(S)\ge \min_{i\in [n]}\acc_i(S)$ for all collaboration settings $S$, there must exist $\alpha\in \{0,1\}$ such that for all $(p_1,p_2,\cdots,p_n)\in (0,1)^n$ where $p_k=\frac{1}{2},$
    \begin{equation}\label{eq:horse}
        \CC(p_1,p_2,\cdots,p_n)=\alpha.
    \end{equation}
\end{proposition}

\begin{proof}
We first establish a collaboration setting $S_1 = (\DD, \AA_1, \cdots, \AA_n)$ such that $\acc_\CC(S_1)=\acc_k(S_1)$ and $\acc_\CC(S_1)<\acc_i(S_1)$ for all $i\neq k.$ Set $\XX = \{0, 1\}$ and choose $\DD$ where
\begin{alignat*}{2}
    \Pr[X=0] &= 1/3, \qquad \qquad \Pr[Y=1\,|\,X=0] &&= \epsilon, \\
    \Pr[X=1] &= 2/3, \qquad \qquad \Pr[Y=1\,|\,X=1] &&= 1 - \epsilon,
\end{alignat*}
where $\epsilon < \frac{1}{2}$.
Now set $\AA_k = \{\{0,1\}\}$ and $\AA_i = \{\{0\}, \{1\}\}$ for all $i\neq k.$ Then $P_k(0)=P_k(1)=\frac{2}{3} - \frac{\epsilon}{3}$, while $P_i(0)=\epsilon$ and $P_i(1)=1-\epsilon$ for $i\neq k.$ Then $\Yhat_\CC(0) = \Yhat_k(0)$ and $\Yhat_\CC(1) = \Yhat_k(1)$ since $(P_1(0), \cdots, P_n(0)), (P_1(1), \cdots, P_n(1))\in (0,1)^n.$
Then observe that $\acc_k(S_1)=\acc_\CC(S_1)=\frac{2}{3}-\frac{\epsilon}{3}$ (since $\CC$ always defers to agent $k$'s classification) and $\acc_i(S_1)=1 - \epsilon$ for $i\neq k$. For $\epsilon < \frac{1}{2}, \frac{2}{3}-\frac{\epsilon}{3} < 1 - \epsilon,$ giving the desired conclusion.

More substantively, we now establish a collaboration setting $S_2$ such that $\acc_\CC(S_2)<\acc_k(S_2).$ If there does not exist $\alpha$ such that $\CC(p_1,p_2,\cdots,p_n)=\alpha$ for all $(p_1,p_2,\cdots,p_n)\in (0,1)^n$ where $p_k=\frac{1}{2}$, then there must exist two tuples $(p_1,\cdots,p_n), (q_1,\cdots,q_n)\in (0,1)^n$ such that $p_k=q_k=\frac{1}{2}$ and
\begin{align}
    \CC(p_1,\cdots,p_n) &= 1\\
    \CC(q_1,\cdots,q_n) &= 0.
\end{align}
Now consider
\begin{equation}
    \XX = \{(0,i):i\in \{0,\cdots,n\}\} \cup \{(1,i):i\in \{0,\cdots,n\}\}.
\end{equation}
Choose $\DD$ such that
\begin{equation}
    \Pr[X=x] = 
    \begin{cases}
        \frac{1}{2 + \sum_{j\in [n]} p_j/(1-p_j) + (1-q_j)/q_j} &\quad \text{for }x\in \{(0,0), (1,0)\}\\
        \frac{p_i/(1-p_i)}{2 + \sum_{j\in [n]} p_j/(1-p_j) + (1-q_j)/q_j} &\quad \text{for }x = (0, i)\\
        \frac{(1-q_i)/q_i}{2 + \sum_{j\in [n]} p_j/(1-p_j) + (1-q_j)/q_j} &\quad \text{for }x = (1, i)
    \end{cases}
\end{equation}
and
\begin{equation}
    \Pr[Y=1|X=x] = 
    \begin{cases}
        0 &\quad \text{for }x=(0,0)\\
        1 &\quad \text{for }x=(0,i): i\in [n]\\
        1 &\quad \text{for }x=(1,0)\\
        0 &\quad \text{for }x=(1,i): i\in [n]
    \end{cases}.
\end{equation}
Then
\begin{align}
    \Pr[Y=1\,|\,X\in \{(0,0),(0,i)\}]&=p_i\\
    \Pr[Y=1\,|\,X\in \{(1,0),(1,i)\}]&=q_i.
\end{align}
Take $\AA_k$ to be the partition with $\{(0,0),(1,0)\}$ and the remaining singleton sets. Take $\AA_i$ for $i\neq k$ to be any partition with $\{(0,0),(0,i)\}$ and $\{(1,0), (1,i)\}$. Then consider $S_2 = (\DD, \AA_1, \cdots, \AA_n).$
Then
\begin{align}
    P_k((0,0))=\frac{1}{2}, \quad P_i((0,i))=p_i, \quad P_k((1,0))=\frac{1}{2}, \quad P_i((1,i))=q_i,
\end{align}
so
\begin{equation}
    \Yhat_\CC(0,0) = \CC(p_1,\cdots,p_n) = 1 = 1 - \Pr[Y=1\,|\,X=(0,0)],
\end{equation}
so $\CC$ is incorrect on $(0,0)$. Similarly, $\CC$ is also incorrect on $(1,0)$. Meanwhile, $k$ is correct on either $(0,0)$ or $(1,0)$, and is correct on all $(0,i)$ and $(1,i)$ for $i\neq 0.$ Therefore, $\acc_\CC(S_2)<\acc_k(S_2).$

The result follows by applying \Cref{prop:linear-combination} with $S_1$ and $S_2$, taking $\lambda_1$ sufficiently close to $1$. In particular, \Cref{prop:linear-combination} implies that for all $\lambda$, there exists a collaboration setting $S$ such that
\begin{equation}
    \acc_i(S) = \lambda \acc_i(S_1) + (1-\lambda) \acc_i(S_2)
\end{equation}
for all $i\in [n]$, and
\begin{equation}
    \acc_\CC(S) = \lambda \acc_\CC(S_1) + (1-\lambda) \acc_\CC(S_2).
\end{equation}
Then, for all $\lambda < 1$, since $\acc_k(S_1)=\acc_\CC(S_1)$ and $\acc_k(S_2)>\acc_\CC(S_2),$ we have that $\acc_k(S) > \acc_\CC(S).$
Now, since $\acc_i(S_1)>\acc_\CC(S_1)$, regardless of $\acc_i(S_2), \acc_\CC(S_2),$ by taking $\lambda$ sufficiently close to $1$, $\acc_i(S) > \acc_\CC(S).$ This provides the desired contradiction.
\end{proof}

Finally, recall that \Cref{thm:main} follows directly from sequentially applying \Cref{prop:part1,prop:part2}. 
\section{Conclusion}\label{sec:conclusion}
In this paper, we proved a ``No Free Lunch''-style result in human-AI collaboration. In particular, in a classification setting with multiple calibrated agents, we showed that any collaboration strategy that is guaranteed to perform no worse than the \textit{least accurate} agent must essentially always defer to the same agent. The result does, however, imply one successful collaboration strategy: deferring to the same agent except when another agent is fully certain in their prediction. More broadly, \Cref{thm:main} suggests that strong individual information (calibration) is not sufficient to enable collaboration; rather, successful collaboration hinges on learning joint information across agents.

\paragraph{Open Problems.}
The present result suggests a number of problems for future work. In the binary setting, it is not clear from the proof of \Cref{thm:main} whether or not complementarity can be achieved for loss functions beyond 0-1 accuracy. For example, the baseline of reliability can be guaranteed under $\ell_2$ loss by simply predicting the average probability of agents. Moreover, the present result places no restrictions on the distribution over $\XX\times \{0,1\}.$ In the spirit of the No Free Lunch Theorem, it would be interesting to consider what restrictions on this distribution enable collaboration strategies (and what those collaboration strategies are). Finally, \Cref{thm:main} does not obviously extend to multi-class classification problems. Multi-class problems further opens up the possibility of collaboration strategies that succeed in set prediction \citep{straitouri2023improving, de2024towards}.

{\small
\bibliographystyle{alpha}
\bibliography{aaai25}

\begin{thebibliography}{22}
\providecommand{\natexlab}[1]{#1}
\providecommand{\url}[1]{\texttt{#1}}
\expandafter\ifx\csname urlstyle\endcsname\relax
  \providecommand{\doi}[1]{doi: #1}\else
  \providecommand{\doi}{doi: \begingroup \urlstyle{rm}\Url}\fi

\bibitem[Alur et~al.(2024)Alur, Raghavan, and Shah]{alur2024distinguishing}
R.~Alur, M.~Raghavan, and D.~Shah.
\newblock Distinguishing the indistinguishable: Human expertise in algorithmic prediction.
\newblock \emph{arXiv preprint arXiv:2402.00793}, 2024.

\bibitem[Arrow(1950)]{arrow1950difficulty}
K.~J. Arrow.
\newblock A difficulty in the concept of social welfare.
\newblock \emph{Journal of Political Economy}, 58\penalty0 (4):\penalty0 328--346, 1950.

\bibitem[Bhatt et~al.(2021)Bhatt, Antor{\'a}n, Zhang, Liao, Sattigeri, Fogliato, Melan{\c{c}}on, Krishnan, Stanley, Tickoo, et~al.]{bhatt2021uncertainty}
U.~Bhatt, J.~Antor{\'a}n, Y.~Zhang, Q.~V. Liao, P.~Sattigeri, R.~Fogliato, G.~Melan{\c{c}}on, R.~Krishnan, J.~Stanley, O.~Tickoo, et~al.
\newblock Uncertainty as a form of transparency: Measuring, communicating, and using uncertainty.
\newblock In \emph{Proceedings of the 2021 AAAI/ACM Conference on AI, Ethics, and Society}, pages 401--413, 2021.

\bibitem[Blum(2005)]{blum2005line}
A.~Blum.
\newblock On-line algorithms in machine learning.
\newblock \emph{Online Algorithms: The State of the Art}, pages 306--325, 2005.

\bibitem[Bondi et~al.(2022)Bondi, Koster, Sheahan, Chadwick, Bachrach, Cemgil, Paquet, and Dvijotham]{bondi2022role}
E.~Bondi, R.~Koster, H.~Sheahan, M.~Chadwick, Y.~Bachrach, T.~Cemgil, U.~Paquet, and K.~Dvijotham.
\newblock Role of human-{AI} interaction in selective prediction.
\newblock In \emph{Proceedings of the AAAI Conference on Artificial Intelligence}, volume~36, pages 5286--5294, 2022.

\bibitem[Breiman(2001)]{breiman2001random}
L.~Breiman.
\newblock Random forests.
\newblock \emph{Machine Learning}, 45:\penalty0 5--32, 2001.

\bibitem[Bu{\c{c}}inca et~al.(2021)Bu{\c{c}}inca, Malaya, and Gajos]{buccinca2021trust}
Z.~Bu{\c{c}}inca, M.~B. Malaya, and K.~Z. Gajos.
\newblock To trust or to think: cognitive forcing functions can reduce overreliance on {AI} in {AI}-assisted decision-making.
\newblock \emph{Proceedings of the ACM on Human-computer Interaction}, 5\penalty0 (CSCW1):\penalty0 1--21, 2021.

\bibitem[Cabrera et~al.(2023)Cabrera, Perer, and Hong]{cabrera2023improving}
{\'A}.~A. Cabrera, A.~Perer, and J.~I. Hong.
\newblock Improving human-{AI} collaboration with descriptions of {AI} behavior.
\newblock \emph{Proceedings of the ACM on Human-Computer Interaction}, 7\penalty0 (CSCW1):\penalty0 1--21, 2023.

\bibitem[de~Condorcet(1785)]{condorcet1785essai}
N.~de~Condorcet.
\newblock Essai sur l'application de l'analyse à la probabilité des décisions rendues à la pluralité des voix.
\newblock 1785.

\bibitem[De~Toni et~al.(2024)De~Toni, Okati, Thejaswi, Straitouri, and Gomez-Rodriguez]{de2024towards}
G.~De~Toni, N.~Okati, S.~Thejaswi, E.~Straitouri, and M.~Gomez-Rodriguez.
\newblock Towards human-{AI} complementarity with predictions sets.
\newblock \emph{arXiv preprint arXiv:2405.17544}, 2024.

\bibitem[Donahue et~al.(2022)Donahue, Chouldechova, and Kenthapadi]{donahue2022human}
K.~Donahue, A.~Chouldechova, and K.~Kenthapadi.
\newblock Human-algorithm collaboration: Achieving complementarity and avoiding unfairness.
\newblock In \emph{Proceedings of the 2022 ACM Conference on Fairness, Accountability, and Transparency}, pages 1639--1656, 2022.

\bibitem[Green and Chen(2019)]{green2019principles}
B.~Green and Y.~Chen.
\newblock The principles and limits of algorithm-in-the-loop decision making.
\newblock \emph{Proceedings of the ACM on Human-Computer Interaction}, 3\penalty0 (CSCW):\penalty0 1--24, 2019.

\bibitem[Kiani et~al.(2020)Kiani, Uyumazturk, Rajpurkar, Wang, Gao, Jones, Yu, Langlotz, Ball, Montine, et~al.]{kiani2020impact}
A.~Kiani, B.~Uyumazturk, P.~Rajpurkar, A.~Wang, R.~Gao, E.~Jones, Y.~Yu, C.~P. Langlotz, R.~L. Ball, T.~J. Montine, et~al.
\newblock Impact of a deep learning assistant on the histopathologic classification of liver cancer.
\newblock \emph{NPJ Digital Medicine}, 3\penalty0 (1):\penalty0 23, 2020.

\bibitem[Lai and Tan(2019)]{lai2019human}
V.~Lai and C.~Tan.
\newblock On human predictions with explanations and predictions of machine learning models: A case study on deception detection.
\newblock In \emph{Proceedings of the Conference on Fairness, Accountability, and Transparency}, pages 29--38, 2019.

\bibitem[Littlestone et~al.(1991)Littlestone, Long, and Warmuth]{littlestone1991line}
N.~Littlestone, P.~M. Long, and M.~K. Warmuth.
\newblock On-line learning of linear functions.
\newblock In \emph{Proceedings of the twenty-third annual ACM Symposium on Theory of Computing}, pages 465--475, 1991.

\bibitem[Madras et~al.(2018)Madras, Pitassi, and Zemel]{madras2018predict}
D.~Madras, T.~Pitassi, and R.~Zemel.
\newblock Predict responsibly: improving fairness and accuracy by learning to defer.
\newblock \emph{Advances in Neural Information Processing Systems}, 31, 2018.

\bibitem[Mozannar et~al.(2022)Mozannar, Satyanarayan, and Sontag]{mozannar2022teaching}
H.~Mozannar, A.~Satyanarayan, and D.~Sontag.
\newblock Teaching humans when to defer to a classifier via exemplars.
\newblock In \emph{Proceedings of the AAAI Conference on Artificial Intelligence}, volume~36, pages 5323--5331, 2022.

\bibitem[Schapire et~al.(1999)]{schapire1999brief}
R.~E. Schapire et~al.
\newblock A brief introduction to boosting.
\newblock In \emph{IJCAI}, volume~99, pages 1401--1406. Citeseer, 1999.

\bibitem[Shin et~al.(2021)Shin, Han, and Rhee]{shin2021ai}
W.~Shin, J.~Han, and W.~Rhee.
\newblock {AI}-assistance for predictive maintenance of renewable energy systems.
\newblock \emph{Energy}, 221:\penalty0 119775, 2021.

\bibitem[Straitouri et~al.(2023)Straitouri, Wang, Okati, and Rodriguez]{straitouri2023improving}
E.~Straitouri, L.~Wang, N.~Okati, and M.~G. Rodriguez.
\newblock Improving expert predictions with conformal prediction.
\newblock In \emph{International Conference on Machine Learning}, pages 32633--32653. PMLR, 2023.

\bibitem[Vaccaro et~al.(2024)Vaccaro, Almaatouq, and Malone]{vaccaro2024combinations}
M.~Vaccaro, A.~Almaatouq, and T.~Malone.
\newblock When combinations of humans and {AI} are useful: A systematic review and meta-analysis.
\newblock \emph{Nature Human Behaviour}, pages 1--11, 2024.

\bibitem[Wolpert and Macready(1997)]{wolpert1997no}
D.~H. Wolpert and W.~G. Macready.
\newblock No free lunch theorems for optimization.
\newblock \emph{IEEE Transactions on Evolutionary Computation}, 1\penalty0 (1):\penalty0 67--82, 1997.

\end{thebibliography}
}

\end{document}